\documentclass{article}

\usepackage[utf8]{inputenc} 
\usepackage[T1]{fontenc}    
\usepackage{hyperref}       
\usepackage{url}            
\usepackage{booktabs}       
\usepackage{amsfonts}       
\usepackage{nicefrac}       
\usepackage{microtype}      
\usepackage{amsfonts}       
\usepackage{nicefrac}       
\usepackage{microtype} 
\usepackage{amsmath, amsthm, thm-restate}
\usepackage{graphicx}
\usepackage{makecell}
\usepackage{tikz}
\usepackage{xcolor}
\usepackage{bbm}
\usepackage{sidecap}
\usepackage{fixltx2e}
\usepackage{makecell}
\usepackage{sidecap}
\usepackage{centernot}
\usepackage{algorithmic}
\usepackage{natbib}
\usepackage{multirow}
\usepackage{color, colortbl}
\usepackage{tikz}
\usepackage{subcaption}
\usepackage{enumitem}
\usepackage{standalone}
\usetikzlibrary{arrows, arrows.meta, positioning}

\usepackage[accepted]{icml2020_noNotice}

\newcommand\sC{\ensuremath{\mathcal{C}}}

\newcommand\BR{\ensuremath{\mathbb{R}}}

\DeclareMathOperator*{\sign}{sign}
\newcommand\pb[1]{\ensuremath{\left[ #1 \right]}} 

\newcommand\R{\ensuremath{\mathbb{R}}} 
\newcommand\refeqn[1]{(\ref{eqn:#1})}

\newcommand\refsec[1]{Section~\ref{sec:#1}}

\newcommand\reffig[1]{Figure~\ref{fig:#1}}

\newcommand\reftab[1]{Table~\ref{tab:#1}}

\newcommand\refprop[1]{Proposition~\ref{prop:#1}}

\newcommand\refcor[1]{Corollary~\ref{cor:#1}}

\ifthenelse{\isundefined{\definition}}{\newtheorem{definition}{Definition}}{}
\ifthenelse{\isundefined{\assumption}}{\newtheorem{assumption}{Assumption}}{}
\ifthenelse{\isundefined{\hypothesis}}{\newtheorem{hypothesis}{Hypothesis}}{}
\ifthenelse{\isundefined{\proposition}}{\newtheorem{proposition}{Proposition}}{}
\ifthenelse{\isundefined{\theorem}}{\newtheorem{theorem}{Theorem}}{}
\ifthenelse{\isundefined{\lemma}}{\newtheorem{lemma}{Lemma}}{}
\ifthenelse{\isundefined{\corollary}}{\newtheorem{corollary}{Corollary}}{}
\ifthenelse{\isundefined{\alg}}{\newtheorem{alg}{Algorithm}}{}
\ifthenelse{\isundefined{\example}}{\newtheorem{example}{Example}}{}
\newcommand{\E}{\ensuremath{\mathbb{E}}} 

\newcommand{\pr}{\ensuremath{\mathbb{P}}}
\newcommand{\pab}[1]{\ensuremath{\left|#1\right|}}

\DeclareMathOperator*{\argmin}{arg\,min}

\newcommand{\eye}{\ensuremath{\text{I}}}
\newcommand{\ms}{\ensuremath{\text{M}^\text{+s}}}
\newcommand{\mnos}{\ensuremath{\text{M}^\text{-s}}}
\newcommand{\mnosST}{\ensuremath{\text{M}^\text{-s}_{\text{RST}}}}

\newcommand{\bstari}{\ensuremath{{{\beta^i}^\star}}}
\newcommand{\bstarj}{\ensuremath{{{\beta^j}^\star}}}
\newcommand{\bstartwo}{\ensuremath{{{\beta^2}^\star}}}
\newcommand{\bstarone}{\ensuremath{{{\beta^1}^\star}}}
\newcommand{\bstar}{\ensuremath{{\beta^\star}}}
\newcommand{\tstar}{\ensuremath{{\theta^\star}}}
\newcommand{\ts}{\ensuremath{\hat\theta^\text{+s}}}
\newcommand{\ws}{\ensuremath{\hat w}}
\newcommand{\tnos}{\ensuremath{\hat\theta^\text{-s}}}
\newcommand{\tnosST}{\ensuremath{\hat\theta^\text{-s}_\text{RST}}}

\newcommand{\coreM}{M^{-s}}
\newcommand{\fullM}{M^{+s}}

\newcommand{\Err}{\text{Error}}
\newcommand{\robErr}{\text{RobustError}}
\newcommand{\maxNorm}{\ensuremath{\gamma}}\newcommand{\sVal}{\ensuremath{\mathcal{S}}}

\newcommand\independent{\protect\mathpalette{\protect\independenT}{\perp}}
\def\independenT#1#2{\mathrel{\rlap{$#1#2$}\mkern2mu{#1#2}}}

\DeclareUnicodeCharacter{00A0}{~}

\icmltitlerunning{Removing Spurious Features can Hurt Accuracy and Affect Groups Disproportionately}

\begin{document}
	
	\twocolumn[
	\icmltitle{Removing Spurious Features can Hurt Accuracy and Affect Groups Disproportionately}



\icmlsetsymbol{equal}{*}

\begin{icmlauthorlist}
	\icmlauthor{Fereshte Khani}{to}
	\icmlauthor{Percy Liang}{to}
\end{icmlauthorlist}

\icmlaffiliation{to}{Department of Computer Science, Stanford University}

\icmlcorrespondingauthor{Fereshte Khani}{fereshte@stanford.edu}

\icmlkeywords{Machine Learning, ICML}

\vskip 0.3in
]
\printAffiliationsAndNotice{}
\begin{abstract}
The presence of spurious features interferes with the goal of obtaining robust models that perform well across many groups within the population.
A natural remedy is to remove spurious features from the model.
However, in this work we show that removal of spurious features can decrease accuracy due to the inductive biases of overparameterized models.
We completely characterize how the removal of spurious features affects accuracy across different groups (more generally, test distributions) in noiseless overparameterized linear regression.
In addition, we show that removal of spurious feature can decrease the accuracy even in balanced dataset--- each target co-occurs equally with each spurious feature;
and it can inadvertently make the model more susceptible to other spurious features. 
Finally, we show that robust self-training can remove spurious features without affecting the overall accuracy. 
Experiments on the Toxic-Comment-Detectoin and CelebA datasets show that our results hold in non-linear models.
\end{abstract}


\section{Introduction}
Machine learning models are vulnerable to fitting spurious features.
For example, models for toxic comment detection assign different toxicity scores to the same sentence with different identity terms (``I'm gay'' and ``I'm straight'') \citep{dixon2018measuring},
and models for object recognition make different predictions on the same object against different backgrounds  \citep{xiao2020noise,ribeiro2016lime}.
A common strategy to make models robust against spurious features
is to remove such features, e.g., removing identity terms from a comment \citep{garg2019counterfactual}, removing background of an image \citep{elhabian2008moving}, or learning a new representation from which it is impossible to predict the spurious feature \citep{zemel2013,louizos2015variational,beutel2017data}.
However, removing spurious features can lower accuracy \citep{zemel2013,wang2019balanced},
and moreover, this drop varies widely across groups within the population \citep{yurochkin2020sensei}.

Previous work has identified two reasons for this accuracy drop:
(i) Core (non-spurious) features are noisy or not expressive enough \citep{khani2020noise,kleinberg2019simplicity},
so spurious features are needed even by the optimal model to achieve the best accuracy.
(ii) Removing spurious features corrupts core features, especially when learning a new representation uncorrelated with the spurious features  \citep{zhao2019inherent}.

In this work, we show that even in the absence of the aforementioned two reasons,
removing spurious features can still lead to a drop in accuracy due to
the dominant effects of inductive bias in overparametrized models. 
For our theoretical analysis, we consider noiseless linear regression:
We have $d$ core features $z$
which determine the prediction target $y = \tstar^\top z$
and the spurious feature $s = \bstar^\top z$.
See \reffig{causal_graph} for the causal graph.
Importantly, (i) $s$ adds no information about $y$ beyond what already exists in $z$,
and (ii) removing $s$ does not corrupt $z$ (since we are not required to remove the correlated features to $s$). 
We consider two models:
the \emph{core model} $\coreM$, which only uses $z$ to predict $y$,
and the \emph{full model} $\fullM$, which also uses the spurious feature $s$.
In this simple setting, one might conjecture that removing the spurious feature should only help accuracy.
However, in this work we show the contrary. 
The main question we wish to answer is:
\emph{For which groups (or more generally, test distributions)
	does removing the spurious feature $s$ help or hurt accuracy?}



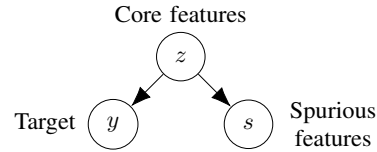
\begin{figure}
\tikzset{
-Latex,auto,node distance =1 cm and 1 cm,semithick,
state/.style ={circle, draw, minimum width = 2em, inner sep=0pt,text width=2em,text centered,},
fstate/.style ={minimum width = 0.5 cm, inner sep=0pt,text width=4em,text centered,},
point/.style = {circle, draw, inner sep=0.04cm,fill,node contents={}},
bidirected/.style={Latex-Latex,dashed},
el/.style = {inner sep=2pt, align=left, sloped},
}
	\centering
\vspace{0pt}
\scalebox{0.9}{
\begin{tikzpicture}[scale=0.5]
\node[fstate,text width=10em] (z) at  (0,-0.75) {Core features};
\node[fstate] (y) at (-4,-4)  {Target};
\node[fstate] (s) at (4.5,-4)  {Spurious features};

\node[state] (z) at  (0,-2) {$z$};
\node[state] (y) at (-2,-4)  {$y$};
\node[state] (s) at (2,-4)  {$s$};

\draw[-{Latex[length=3mm]}] (z)--(y);
\draw[-{Latex[length=3mm]}] (z)--(s);
\end{tikzpicture}}
\caption{
\label{fig:causal_graph}
We compare two models:
the \emph{core model} $\coreM$, which predicts $y$ from only the core features $z$;
and the \emph{full model} $\fullM$, which predicts $y$ from both $z$ and the spurious feature $s$.
}
\end{figure}

\begin{table*}
\newcommand{\sdd}{0.8}
\scalebox{\sdd}{
\begin{tabular}{c|lc}\toprule
\multicolumn{3}{c}{True Parameters}\\ \midrule
$\tstar=[2,2]^\top$&target&$y=[2,2] z$\\ 
$\bstar=[1,\alpha]^\top$&spurious&$s=[1,\alpha] z$\\ \bottomrule
\multicolumn{3}{c}{(a)}
\end{tabular}}\hfill%
\scalebox{\sdd}{
\begin{tabular}{cc|c}\toprule
\multicolumn{3}{c}{Training Example} \\ \midrule
$z$&\small $s$ & \small $y$\\ $[1,0]^\top$&1&2\\
\bottomrule 
\multicolumn{3}{c}{(b)}
\end{tabular}}\hfill%
\scalebox{\sdd}{
\begin{tabular}{ll}\toprule
\multicolumn{2}{c}{Model estimated parameters} \\ \midrule
Core model&$\hat \theta = [2,0]^\top$\\
Full model
&$\hat \theta = [1,0]^\top, w=1$
\\ 
\bottomrule
\multicolumn{2}{c}{(c)}
\end{tabular}}\hfill%
\scalebox{\sdd}{
\begin{tabular}{ll}\toprule
\multicolumn{2}{c}{Model prediction} \\ \midrule
Core model & $\hat y = [2,0] z$ \\ 
Full model & $\hat y = [1,0] z + s = 
[2, \alpha] z$\\ 
\bottomrule
\multicolumn{2}{c}{(d)}
\end{tabular}}
\caption{
\label{tab:intro_example}
A simple linear regression example that shows when removing a spurious feature increases the error. 
(a) There are 2 core features, $z = [z_1, z_2]$, one spurious feature $s = \bstar^\top z$ and the target is $y=\tstar^\top z$ . 
(b) There is only one training example with no information about $z_2$. 
(c) Core model's estimate for the parameters is $[2,0]$; 
however, the full model interpolates data with a smaller norm by putting weight $w=1$ for the spurious feature $s$. 
(d) For the prediction, we can replace $s$ by $\bstar^\top z=[1,\alpha]^\top z$,  which results in the full model implicitly assigning weight of $\alpha$ for the second feature; while the core model assigns wight of $0$ to the second feature.
As a result, if $|\alpha-2| < |0-2|$  then removing $s$ increases the error. 
}
\end{table*}

In the overparametrized regime, since number of training examples is less than the number of features, there are some directions of data variation that are not observed in the training data.
For instance, \reftab{intro_example} shows a simple example with only one training data and two core features ($z_1$ and $z_2$). 
In this example,  we do not observe any information about the second feature ($z_2$).
The core model assigns weight $0$ to the unseen directions (weight $0$ for the second feature in \reftab{intro_example}(c)).
On the other hand, in the presence of a spurious feature, the full model can fit the training data perfectly  with a smaller norm by assigning a non-zero weight for the spurious feature (weight $1$ for the feature $s$ in \reftab{intro_example}(c)).
This non-zero weight for the spurious feature leads to a different inductive bias for the full model.
In particular, the full model does not assign weight $0$ to the unseen directions.
In \reftab{intro_example}(d), the full model \emph {implicitly} assigns  weight $\alpha$ to the second feature (unseen direction at training), while the core model assigns weight $0$. 
As a result, if $|\alpha-2|$ is small, then the full model which uses $s$ has lower error, and if $|\alpha -2|$ is large then core model which does not use $s$ has lower error.

Intuitively, by using the spurious feature, the full model incorporates $\bstar$ into its estimate. 
The true target parameter ($\tstar$) and the true spurious feature parameters ($\bstar$) agree on some of the unseen directions and do not agree on the others. 
Thus, depending on which unseen directions are weighted heavily in the test time, removing $s$ can lower or higher the error.
We formalize the conditions under which the removal of the spurious feature ($s$) increases the error in \refprop{whenRemoveSHurts}.
In particular, \refprop{whenRemoveSHurts} states that in addition to the true parameters (\tstar, \bstar), the distribution of core-features in train and test data are also critical on the impact of removing $s$ on error. 
Consequently, conditions only on the true parameters (e.g., the target and spurious feature are determined by disjoint features, $\theta^\star_i\beta^\star_i=0$, for all $i$), or only on the spurious feature and target (e.g., balanced dataset, $y \independent s$ in the train and test data) are not sufficient, and removing the spurious feature may aggravate the error even under these favorable conditions.

We then study multiple spurious features and show that removing one spurious feature inadvertently makes the model more susceptible to the rest of the spurious features, in line with the recent empirical results \citep{yin2019fourier}.
Finally, we show how to leverage unlabeled data and the recently introduced Robust Self-Training (RST) \citep{raghunathan2020understanding,carmon2019unlabeled,najafi2019robustness} to remove the spurious features but retain the same performance as the full model.
The new model, Core+RST model, is robust to changes in $s$, and it can perform when $s$ is not available.


Empirically, we analyze the effect of removing spurious features by training a convolutional neural network on three datasets:
the CelebA dataset for predicting if a celebrity is wearing lipstick where we use wearing earrings as the spurious feature;
the Comment-Toxicity-Detection dataset for predicting the toxicity of a comment where we use identity terms as the spurious features; and 
finally, a synthetically-generated dataset where we concatenate each MNIST image with another image and use the label of the new image to be the spurious feature.
Our empirical results are four folds:
1) Removal of the spurious feature lowers the average accuracy and disproportionately affects different groups (+30\% increase for some groups and -7\% decrease for others). 
2) The full model is not robust against the spurious feature, and changing the spurious feature at the test time lowers its accuracy substantially. 
3) The Core+RST achieves similar average accuracy as the full model while being robust against the spurious feature.
4) In the CelebA dataset, we show that removing the spurious hair color feature makes the model less robust against wearing earrings. 


\newcommand{\xc}{\ensuremath{x_{\text{core}}}}
\newcommand{\xs}{\ensuremath{x_\text{sp}}}
\section{Setup}
Let $z \in \BR^{d}$ denote the core features which determine the prediction target, $y = f(z)$.
Let $s = g(z)$ denote a spurious feature.
We study the overparameterized regime, where we observe $n < d$ triples ($z_i, s_i, y_i)$ as training data. 
For an arbitrary loss function $\ell$, the standard error for a model $M:\BR^d \times \BR \rightarrow \BR$ is:
\begin{align}
\Err (M) = \E [ \ell (y, M(z,s))],
\end{align}
where the expectation is over test points $(z, s, y)$.
We are also interested in robustness of a model against spurious feature.
Let \sVal{} denote the set of different values for the spurious feature.
We define the robust error as follows:
\begin{align}
\label{eqn:robErr}
\robErr(M) = \E [ \max_{s’ \in \sVal} \ell (y, M(z,s’))]
\end{align}
\robErr{} measures the worst case error for each data point $z$ with respect to change in $s$. 
Ideally we want our model prediction to be robust with respect to change in the spurious feature (i.e., $\Err (M) = \robErr (M)$).
\robErr{} is a common definition in robust machine learning against input perturbation~\cite{carmon2019unlabeled,raghunathan2020understanding,ilyas2019adversarial}.
Also \robErr{} is close to counterfactual notions of fairness, such as Counterfactual Token Fairness \cite{garg2019counterfactual},  see \refsec{related_work} for more discussion. 

Let $\sC$ denote a function that measures the complexity of a model (e.g, the $L2$-norm of its parameters), and $\lambda$ be a parameter to tune the trade-off between complexity of a model and its empirical risk.
We are interested in the performance of the following two models:
\begin{itemize}
	\item {\bf Full model (\ms):} which uses the core features $z$ and the spurious $s$ to predict $y$.
	\begin{align}
		\label{eqn:ms_setup}
	\ms = \argmin_{M} \sum_i \ell \left (M(z_i,s_i),y_i\right) + \lambda \sC(M)
	\end{align}
	\item {\bf Core model (\mnos):} which only uses the core features $z$ to predict $y$.
		\begin{align}
			\label{eqn:mnos_setup}
	\mnos = \argmin_M \sum_i \ell (M(z_i,0),y_i) + \lambda \sC (M)
	\end{align}
\end{itemize}
The core model does not use the spurious feature, therefore its \robErr{} is equal to its \Err.  However, there might be a large gap between \Err{} and \robErr{} for the full model.
 

\begin{table*}[t]
\newcommand{\sdd}{0.8}
\centering	

\scalebox{\sdd}{
\begin{tabular}{lc}\toprule
\multicolumn{2}{c}{True parameters}\\ \midrule
$\tstar=$&$[2,2,\ \ 2]^\top$\\ 
$\bstar=$&$[1,2,-2]^\top$ \\ \bottomrule
\end{tabular}}\hfill
\scalebox{\sdd}{
\begin{tabular}{ccc}\toprule
	\multicolumn{3}{c}{Training example}\\ \midrule
$z$&$s$ &  $y$\\ 
$[1,0,0]^\top$&1&2\\
\bottomrule
\end{tabular}}\hfill
\scalebox{\sdd}{
\begin{tabular}{ll}\toprule
	\multicolumn{2}{c}{Model estimated parameters}\\ \midrule
Core model &$\tnos = [2,0,\ \ 0]^\top$\\ 
Full Model
&$\ts = [1,0,\ \ 0]^\top, w=1$ 
\\ \bottomrule
\end{tabular}}\hfill
\scalebox{\sdd}{
\begin{tabular}{ll}\toprule
	\multicolumn{2}{c}{Model prediction} \\ \midrule
	Core model & $\hat y = [2,0,0] z$ \\ 
	Full model & $\hat y = [1,0,0] z + s = 
	[2, 2,-2] z$\\ 
	\bottomrule
\end{tabular}}
\caption{
\label{tab:example2}
There is only one training example with no information about $z_2$ and $z_3$. 
The core model assigns weight of $0$ to these two features.
The full model, uses the spurious feature ($s$),  which results in a good estimation for real weight of $z_2$ but not $z_3$.
}
\end{table*}

\section{When does the Full Model Outperform the Core Model?}
\label{sec:ones}
For the theory section of the paper, we consider the noiseless linear regression setup.
We assume there are true parameters $\tstar, \bstar \in \BR^d$ such that the
prediction target, $y = \tstar^\top z$, and the spurious feature $s = \bstar^\top z$ (in \refsec{multis} we study multiple spurious features and their interactions).
Motivated by recent work in deep learning, which shows gradient descent converges to the minimum-norm solution that fit training data perfectly \cite{gunasekar2017implicit}, we consider the minimum-norm solution (i.e., the complexity function $\sC$ in \refeqn{ms_setup} and \refeqn{mnos_setup} returns the $L_2$-norm of the parameters with regularization strength tending to $0$).
We consider squared-error for the loss function.

Let $Z \in \BR^{n\times d}$ be the matrix of observed core features, $Y \in \BR^n$ be the targets, and $S \in \BR^n$ be the spurious features. 
We first analyze the minimum-norm estimate for the core model. 
Let $\tnos$ denote the estimated parameters of the core model which can be obtained through solving the following optimization problem. 
\begin{align}
	\label{eqn:nos}
	\tnos = \argmin_{\theta} \| \theta\|_2^2 \nonumber\\
	s.t. \quad Z\theta =Y.
\end{align}

Let $\Pi = Z^\top (ZZ^\top)^{-1}Z$ be the projection matrix to columns of $Z$ then $\tnos = \Pi \tstar$. 

The full model uses the spurious feature in addition to the core features.
Let $\ts$ and $\ws$ denote the estimated parameters for the full model which can be obtained through solving the following optimization problem.
\begin{align}
	\label{eqn:withs}
	\ts, \ws = \argmin_{\theta, w} \| \theta\|_2^2 + w^2 \nonumber\\
	s.t. \quad Z\theta + Sw =Y,
\end{align}
where $w$ denote the weight associate with the spurious feature.
Note that in \refeqn{withs} we can always set $\ws = 0$ to obtain $\tnos$, so the full model only achieves smaller norm by optimizing over $w$.
In particular, instead of having norm of $\|\Pi \tstar\|_2^2$, the full model can use $s$ with weight $\ws$ and have norm $\|\Pi\tstar - \ws\Pi\bstar\|_2^2 + \ws^2$ instead.
As a result, $\ws$ is larger if $\tstar$ and $\bstar$ are correlated in column space of the training data ($\Pi$).
The optimum value for $\ws$ which minimizes the norm is:
\begin{align}
	\label{eqn:ws}
	\ws = \frac{\tstar^\top \Pi \bstar}{1 + \bstar^\top \Pi \bstar} 
\end{align}

See Appendix~\ref{sec:proofs} for details.
The estimated parameters for the full and core models are shown in \reftab{eq_summary} (first and second rows).



To understand the difference between inductive biases of the core model and full model and how they affect different groups, we extend the example in the introduction to contain two unseen directions, see \reftab{example2}.
In this example, there are $d=3$ core features and only one training example with no information about the second and the third features.
Let $\tstar =
[2,2,2]^\top$ which results in $y=2$.
Let $\bstar = [1,2,-2]^\top$ which results in $s=1$. 
Without the spurious feature, the estimated parameter is $\tnos = [2,0,0]^\top$ however, by using $s$ the full model fits the training data perfectly with a smaller norm by setting $\ws=1$. 
By substituting $s$ with $\bstar^\top z$, the full model implicitly assigns weights of $2$ and $-2$ to the second and third features, respectively. 
Therefore, in this example, removing $s$ decreases the error for the groups with high variance on the second feature, and removing $s$ increases the error for the groups with high variance on the third feature.


%
The following proposition, which is our main result, provides general conditions that characterize precisely when the core model (removing the spurious feature) increases the error over the full model. 

\begin{restatable}{proposition}{biasVarianceMsMnos}
  Let $\Sigma = \E[z z^\top]$ denote the covariance matrix of the test data (or covariance matrix for any group), and let $\Pi$ denote the column space of training data. 
	$\Err (\ms) < \Err (\mnos)$ iff:
	\begin{align}
		\label{eqn:sign}
		\sign \left({{\bstar}^\top \Pi \tstar}\right)=\sign \left ({\bstar}^\top (I-\Pi)\Sigma (I-\Pi) \tstar\right), 
	\end{align}
and
	\begin{align}
	\label{eqn:magn}
	\pab {\frac{{\bstar}^\top \Pi \tstar}{{1+\bstar}^\top \Pi \bstar}} < \pab {\frac{2{\bstar}^\top (I-\Pi)\Sigma (I-\Pi) \tstar}{{\bstar}^\top (I-\Pi) \Sigma (I-\Pi) \bstar}}. 
\end{align}
	\label{prop:whenRemoveSHurts}
\end{restatable}
Intuitively, this proposition states that removing the spurious feature increases the error if the correlation between $\bstar$ and $\tstar$ in column space of training data (seen directions) is similar to the correlation of $\bstar$ and $\tstar$ in the null-space of training data (unseen directions) scaled by the covariance matrix. 
Applying \refeqn{magn} to the example in \reftab{example2}, we see that removing $s$ reduces the error at the test time with covariance matrix $\Sigma$ if $2\Sigma_{23}+3\Sigma_{22} \le \Sigma_{33}$. 
This is in line with our intuition since the full model recovers $\theta^\star_2$ exactly, however, its estimation for $\theta^\star_3$ is worse than the core model's estimate. 
\begin{table*}[ht]
	\newcommand{\clr}{\cellcolor{cyan!25}}

\begin{minipage}{0.35\textwidth}
\centering	

\begin{tabular}{lc}\toprule
\multicolumn{2}{c}{True parameters}\\ \midrule
$\tstar=$&$[1,0,\ \ \ 1,\ \ \ 0]^\top$\\ 
$\bstar=$&$[1,1,-1,-1]^\top$\\ \bottomrule
\end{tabular}

\vspace{0.5cm}

\begin{tabular}{ccc}\toprule
	\multicolumn{3}{c}{Training examples}\\ \midrule
$z$ &\clr $s$ & \clr $y$\\
$[1,0,0,0]$&\clr 1&\clr 1\\
$[0,1,0,0]$&\clr 1&\clr 0\\
\bottomrule
\end{tabular}
\end{minipage}\hspace{0.1\textwidth}%
\begin{minipage}{0.45\textwidth}
	\centering
\begin{tabular}{llccl}\toprule
\multicolumn{5}{c}{Estimated parameters}\\ \midrule
$\tnos=$ &$[1$,&$0$,&$0$,&$0]^\top$\\ 
$\ts=$
&$[2/3$,&$-1/3$,&$0$,&$0]^\top, w=1/3$\\ 
\bottomrule
\end{tabular}
\\\vspace{0.5cm}
\begin{tabular}{ccc|cc}\toprule
	\multicolumn{3}{c|}{Test Examples} & 	\multicolumn{2}{c}{Models predictions}\\ \midrule
	$z$ &\clr$s$ & \clr $y$&Core model (\mnos)&Full model (\ms)\\ 
	$[0,2,1,0]$&\clr 1&\clr 1&0&-1/3\\
	$[0,2,0,1]$&\clr 1&\clr 0&0&-1/3\\
	\bottomrule
\end{tabular}
\end{minipage}
\caption{
\label{tab:same_s_y}
An example demonstrating that $s$ and $y$ can be exactly the same in train and test time (blue filled cells), but the full model (which uses $s$) has a worse performance in comparison to the core model (which does not use $s$). 
}
\end{table*}

\begin{table*}[ht]
\let\footnoterule\relax

\begin{tabular}{llrr}\toprule
	\multicolumn{4}{c}{True parameters}\\ \midrule
$\tstar=$&$[2,$&\cellcolor{red!25}$2,$&$ 2]^\top$\\ 
$\bstarone=$&$[1,$&$-3,$&$ 0]^\top$\\
$\bstartwo=$&$[1,$&$ 0,$&$-3]^\top$\\  \bottomrule
\end{tabular}\hfill
\begin{tabular}{cccc}\toprule
	\multicolumn{4}{c}{Training example}\\ \midrule
$z$& $s_1$ &$s_2$ & \small $y$\\ 
$[1,0,0]$&1&1&2\\
\bottomrule
\end{tabular}\hfill
\begin{tabular}{llcrcc}\toprule
\multicolumn{6}{c}{Estimated parameters}\\ \midrule
&\multicolumn{3}{c}{$\theta$}& $w_1$ & $w_2$ \\ \midrule
with $s_1$ and $s_2$ &$[2/3,$ &$0,$ & $0]^\top$&\cellcolor{cyan!25}$2/3$&$2/3$\\ 
(equivalently) & $[2,$&\cellcolor{red!25}$-2,$&$-2]^\top$&$0$&$0$\\ \midrule
with only $s_1$
&$[1,$ & $0,$ & $0]^\top$&\cellcolor{cyan!25}$1$&$0$\\ 
(equivalently) & $[2,$ &\cellcolor{red!25}$-3,$ & $0]^\top$&$0$&$0$\\ \bottomrule
\end{tabular}
\caption{
\label{tab:multis}
An example with two spurious features ($s_1$ and $s_2$).
There are two models: one model which only uses $s_1$ and another model which uses $s_1$ and $s_2$.
Removing $s_2$ increases the weight for $s_1$ (weight $1$ in comparison to $2/3$, blue filled cells). 
As a result the model that only uses $s_1$ is more susceptible to change in $s_1$ (i.e., increase the \robErr{} with respect to $s_1$) and it performs worse on groups with high variance on $z_2$ (weight of $-3$ instead of $-2$ while the true weight is $2$, red filled cells).
}
\end{table*}

\subsection{Disjoint Parameters and Balanced Dataset  are not Enough}
\refprop{whenRemoveSHurts} characterizes when removal of the spurious feature decreases the overall error.
Now we investigate if removing spurious feature always decreases the error for some special settings. 
For example, if the features that determine the spurious feature ($s$) and the target ($y$) respectively are disjoint (i.e., for all $i$ we have: $\beta_i^\star \theta_i^\star = 0$), or when spurious features and the target are independent (i.e., $s \independent y$ in the train and test empirical distribution). 

The following corollary states that there is no condition on the true parameters that guarantees error reduction by removing the spurious feature. 
\begin{restatable}{corollary}{needZone}(Disjoint parameters are not enough)
	\label{cor:needZone}
	For $d \ge 4$, consider any non-zero $\tstar,\bstar \in \BR^d$, such that there is no scalar $c$ where $\bstar = c\tstar$.
	For any $n<d-1$, we can construct $Z\in \BR^{n \times d}$ as training and $,Z',Z'' \in \BR^{n \times d}$ as test data such that if we train $\ms$ and $\mnos$ using $Z$ as training data, then $\Err (\ms) < \Err (\mnos)$ on $Z'$, and $\Err (\ms) > \Err (\mnos)$ on $Z''$. 
	On the contrary, if $\bstar = c \tstar$, then training both models on any $Z$, results in $\Err (\ms) \le \Err (\mnos)$ on any $Z'$.
\end{restatable}
For any $\bstar$ and $\tstar$, we can choose the training data $Z$, such that the parameters have positive correlation in the column space of training data.
We can then select the test data in a way that the correlation between \tstar{} and \bstar{} in the null space has the same or opposite sign. 
See Appendix~\ref{sec:proofs} for the full proof.
Note that even in the special case of disjoint parameters ($\theta_i^\star\beta_i^\star =0$, for all $i$), removing $s$ can increase the error depending on the core features distribution.

Can we guarantee error reduction by conditioning on the vector of observed spurious features and the targets in train and test data?
In fact, one of the proposed ways to reduce the sensitivity of the model against the spurious features is having a balanced dataset.
(i.e., collecting the dataset such that $y$ and $s$ be independent in train and test data, $\pr [y \mid s] = \pr [y]$).
For example, in comment toxicity detection, \citet{dixon2018measuring} suggest adding new examples in training data to equalize the number of toxic comments and non-toxic comments for each identity term. 
Although they show mitigation with this method, \citet{wang2019balanced} demonstrate that a balanced dataset is not enough, and the model can still be sensitive to the spurious features.

We now show that no measure depending on the spurious features and the target in train and test can guarantee that removing a spurious feature decreases the error.

\begin{restatable}{corollary}{needZtwo}(Balanced dataset is not enough)
	\label{cor:needZtwo}
	Consider any $d \ge 4$, $n < d$ and $S,Y\in \R^n$, where there is no scalar $c$ such that $Y=cS$. 
	We can construct $Z,Z',Z'' \in \R^{n\times d}$, such that  if we use $(Z,S,Y)$ as training set to train $\ms$ and $\mnos$, then the $\Err(\ms) < \Err (\mnos)$ on $(Z',S,Y)$ and $\Err(\mnos) < \Err (\ms )$ on $(Z'',S,Y)$.
	On the contrary, if $Y=cS$ then training both models on any $(Z,S,Y)$, results in $\Err (\ms ) \le \Err (\mnos)$ on any $(Z',S,Y)$. 
\end{restatable}
See Appendix~\ref{sec:proofs} for the proof. At a high level, for proving this corollary, we first rewrite the formulation of $\ws$ in terms of $S$ and $Y$. 
\begin{align}
	\ws = \frac{\tstar^\top \Pi \bstar}{1 + \bstar^\top \Pi \bstar} = \frac{S^\top (ZZ^\top)^{-1}Y}{1 + S^\top (ZZ^\top)^{-1}S}
\end{align}
We then construct $Z$ such that the dot product of $S$ and $Y$ projected on $(ZZ^\top)^{-1}$ is non-zero. 
We then construct $Z',Z'',\tstar,\bstar$ such that $Y=Z\tstar=Z'\tstar=Z''\tstar$, and $S=Z\bstar=Z'\bstar=Z''\bstar$. 
Furthermore, we select $Z'$ to have high variance on the unseen directions that \tstar{} and \bstar{} have negative correlation, while choosing $Z''$ to have high variance on the unseen directions with positive correlation between \tstar{} and \bstar{}. See Appendix~\ref{sec:proofs} for the proof.

There are two main messages in \refcor{needZtwo}, (i) The core model can have lower error than the full model even when there is no shift in the distribution of the spurious feature and target (i.e., $S$ and $Y$ are exactly the same in train and test).
See \reftab{same_s_y} for a simple example.
(ii)
We can have independent targets and spurious features at train and the test, and still, the full model (which uses $s$) has lower error than the core model.

\refcor{needZone} and \ref{cor:needZtwo} together state that we cannot compute the sensitivity of a model to the spurious feature, or the effect of removing spurious feature by only observing the relation between $s$ and $y$ or knowing the relationship between \tstar{} and \bstar{}.
Therefore, naively collecting a balanced dataset is not enough, and we need to consider other features as well.
\subsection {\robErr{} analysis} 
 \refprop{whenRemoveSHurts} compares the error of the full model and core model, and show each model can outperform the other on some conditions.
In this section, we show that \robErr{} of the full model is always larger than the \robErr{} of the core model (recall that for the core model $\robErr(\mnos) = \Err (\mnos)$).
Intuitively for any $(z,s)$ data point we can perturb $s$ such that the full model makes the same prediction as the core model, therefore, the full model's \robErr{} is always lower than the core model error.

\begin{restatable}{proposition}{robustErrorCompared}
If $\|z\| \le \maxNorm$ and consequently the set of different values for the spurious feature is $\sVal = [-\maxNorm\|\bstar\|_*, \maxNorm\|\bstar\|_*]$ then:
\begin{align}
	\robErr(\mnos)  \le \robErr(\ms)
\end{align}
\end{restatable}

Note that without any bounds on the spurious feature the \robErr{} can be infinity.
By bounding the norm of $z$, we can bound the perturbation set of the spurious feature. Here $\|.\|_*$ indicates the dual norm, see Appendix~\ref{sec:proofs} for the proof.

\subsection{Multiple spurious features}
\label{sec:multis}
We now extend our framework to $k$ spurious features, $s_1,\dots,s_k$, where $s_i = {\bstari}^\top z$.
We characterize the effect of removing one spurious feature on the other spurious features.
In particular, we show that removing one spurious feature can make the model more susceptible to other spurious features.

Extending the notation from \refsec{ones}, let $Z \in \BR^{n\times d}$ be the core features, $Y \in \BR^n$ be the target, and $S \in \BR^{n \times k}$ denote the spurious features. 
Similar to previous section, we obtain the minimum-norm estimate by solving the following optimization problem:
\begin{align}
	\label{eqn:multis_opt}
	\ts, \ws = \argmin_{\theta, w} \| \theta\|_2^2 + \|w\|^2_2 \nonumber\\
	s.t. \quad Z\theta + Sw =Y,
\end{align}
where $\ws \in \BR^k$ denote the optimal weights on spurious features. 
\begin{restatable}{proposition}{multiSFeatures} The weight of the $i^\text{th}$ spurious feature of the minimum norm estimator is
	\begin{align}
		\label{eqn:multis}
		\ws_i = \frac{\tstar^\top\Pi \bstari - \sum\limits_{j\neq i} \ws_j {\bstari}^\top \Pi \bstarj}{1+{\bstari}^\top \Pi \bstari}.
	\end{align}
\end{restatable}


Consider the special case of $k=2$ with $s_1$ and $s_2$ as spurious features, where $\bstarone$ and $\bstartwo$ are positively correlated in the column space of training data.
\reftab{multis} shows a simple example of this setup.
As shown in \refeqn{multis}, removing $s_2$ increases the weight for $s_1$, which makes the model more sensitive against changes in $s_1$.
For instance, in \reftab{multis}, after removing $s_2$ the weight of $s_1$ changed from $2/3$ to $1$ (green cells).

In addition, recall that using $s_1$ causes high error for groups with high variance on the unseen directions where $\bstarone$ differs from $\tstar$.
Removing $s_2$ exacerbates the problem even more in these groups.
In our example, using $s_1$ and $s_2$ the model's estimate for $\theta_2^\star$ is $-2$ which is different from the true value of $\theta_2^\star = 2$; therefore, groups with high variance on the second feature incur high error. 
Removing $s_2$ changes the estimate of $\theta_2^\star$ from $-2$ to $-3$, which exacerbates the error on the groups with high variance on $z_2$.
This is in line with the recent empirical results suggesting that making a model robust against one type of noise might make it more vulnerable against other types of noise \citep{yin2019fourier}.


\begin{table*}
	\centering
	\scalebox{0.9}{
	\begin{tabular}{ccll} \toprule
		model & weight for $s$  & weights for $z$ & Error\\ \midrule
		\mnos & $0$ & $\Pi \tstar$ & $\E [(y-z^\top\Pi\tstar)^2]$\\ 
		\ms& $w$  &$\Pi \tstar - w\Pi \bstar$  &$\E[ (y-z^\top\Pi\tstar)^2] + \E [w^2 (s-z^\top\Pi\bstar)^2] \quad\quad-  \E[ 2w (y -z^\top\Pi\tstar) (s-z^\top\Pi\bstar)]$ \\ 
		\mnosST & $0$ & $\Pi \tstar + w(    \eye-\Pi)\bstar$ & $\E[ (y-z^\top\Pi\tstar)^2] + \E [w^2 (z^\top\bstar - z^\top\Pi\bstar)^2] - \E[ 2w (y -z^\top\Pi\tstar)(z^\top\bstar - z^\top\Pi\bstar)]$ \\ \bottomrule
	\end{tabular}}
	\caption{\label{tab:eq_summary} Estimated parameters and error for different models. $w=\frac{{\bstar}^\top \Pi \tstar}{{1+\bstar}^\top \Pi \bstar}$}
\end{table*}
\section{Self-Training}
\label{sec:rst}



In the previous examples shown in \reftab{intro_example} and \reftab{example2}, we showed that the full model which uses the spurious feature (\ms) has an equivalent form which has weight $0$ on the spurious feature. 
This equivalent form is useful when feature $s$ is not available, or when we test our model in a new domain with different $(z,s)$ distribution.
We now explain how to recover the equivalent $s$-oblivious form with finite labeled data and access to $m>d$ independent unlabeled examples.

In order to recover the $s$-oblivious equivalent form of the full model, we use Robust Self Training (RST).
RST introduced by \citet{carmon2019unlabeled,raghunathan2020understanding,najafi2019robustness,uesato2019are} leverage unlabeled data to make the model robust to adversarial perturbations (small perturbation on the model's input) while guaranteeing no decrease in average accuracy. 
In particular, RST first pseudo label all unlabeled data, with a non-robust model (with low error). 
It then trains a robust model on labeled data and pseudo labeled data to recover a model with both low error and low \robErr.

We now explain how to use RST in our setup. 
Assume in addition to $n$ labeled examples ($z_i,s_i,y_i$), we have access to $m$ unlabeled examples $(z^u_i, s^u_i)$.
We are interested in a model that 1) it is robust against $s$, and 2) it has the same prediction as full model $\ms$ on unlabeled data. 
We recover $\mnosST$ as follows:
\begin{align}
	\label{eqn:self_training_general}
	\mnosST = \argmin_M \Big (&\sum \ell (M (z_i,0), y_i)\nonumber\\
	&+\eta \sum \ell (M(z^u_i,0), \ms (z^u_i,s^u_i)) \nonumber\\
	&+\lambda \mathcal{C}(M)\Big),
\end{align}
where $\lambda,\eta$ are parameters to tune the trade-off between the true labeled data, pseudo-labeled data generated by the full model, and complexity of the model.

We now show that in our linear regression setup, $\mnosST$ is a $s$-robust equivalent form of $\ms$. 
Let $Z^u \in \BR^{m \times d}$ and $S^u \in \BR^{m}$ denote the unlabeled data. Let $\tnosST$ denote the estimated parameters of RST model obtained by solving the following optimization problem.
\begin{align}
	\label{eqn:rst_comp_prop4}
	\tnosST = &\argmin_\theta \|\theta\|_2^2 \nonumber\\
	\text{s.t}\quad &Z\theta = Y \nonumber \\
	&Z^u \theta = Z^u \ts + S^uw 
\end{align}
The following proposition shows the optimum parameters for \refeqn{rst_comp_prop4} and prove that for any data point $\mnosST$ has the same prediction as $\ms$.
\begin{restatable}{proposition}{rst}
	\label{prop:rst}
The optimum parameters for \refeqn{rst_comp_prop4} are:
\begin{align}
	\ws &=\frac{{\bstar}^\top \Pi \tstar}{{1+\bstar}^\top \Pi \bstar}\\
	\tnosST &= \Pi \tstar + w(I-\Pi) \bstar
\end{align}
and for any data point $(z,s)$, we have:
\begin{align}
	\mnosST (z,0) = \ms (z,s)
\end{align}
\end{restatable}

See Appendix~\ref{sec:proofs} for details. 
Note that \mnosST{} simply learn $\bstar$ from unlabeled data and replace $ws$ with $w\bstar^\top z$. 
See \reftab{eq_summary} for the estimated parameters of the three introduced models.


\begin{table*}[t]
	\centering
	\scalebox{0.9}{
	\begin{tabular}{p{2.5cm}p{2.5cm}p{2.5cm}p{2.7cm}|p{3.7cm}cc}\toprule
name & core features ($z$) & target ($y$) & spurious feature ($s$) & Example:  z & y & s \\ \midrule
Double-MNIST & two MNIST images & label of the left image & label of the right image & \centering \raisebox{-.5\height}{ \includegraphics[height=0.07\textwidth]{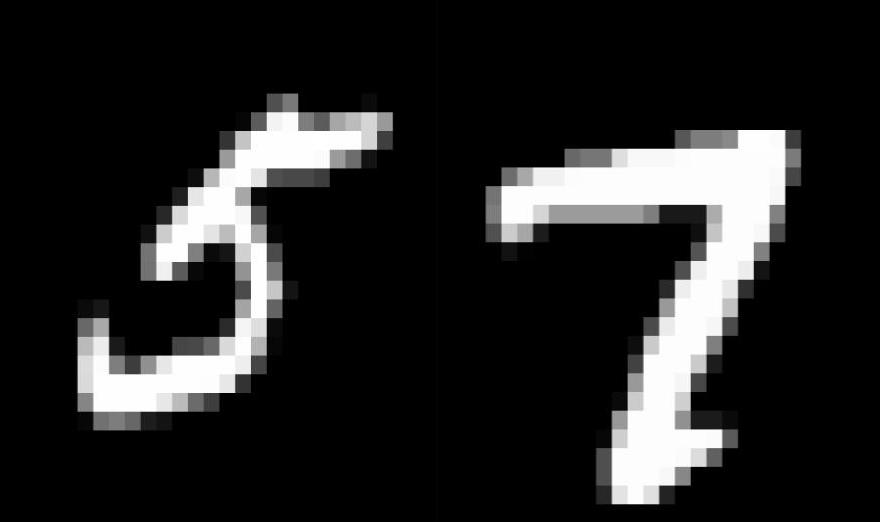}} & 5 & 7 \\ \midrule
CelebA & celebrities photo & wearing lipstick & wearing earrings &\centering\raisebox{-.5\height}{\includegraphics[height=0.07\textwidth]{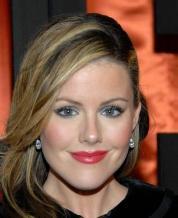}} & True & True \\ \midrule
Toxic-Comment-Detection & comment (w/o identity terms) & toxic or not & identity terms
&cuz i shouldn't be blocked just for being UNK& non-toxic & black \\
 \bottomrule
	\end{tabular}}
  \caption{\label{tab:datasets} A summary of the three datasets that we used in this work. 
Double-MNIST is a synthetically generated dataset where each data point is a concatenation of two images from the MNIST dataset.
  }
\end{table*}

\begin{figure}
	\centering
\includegraphics[width=0.5\textwidth]{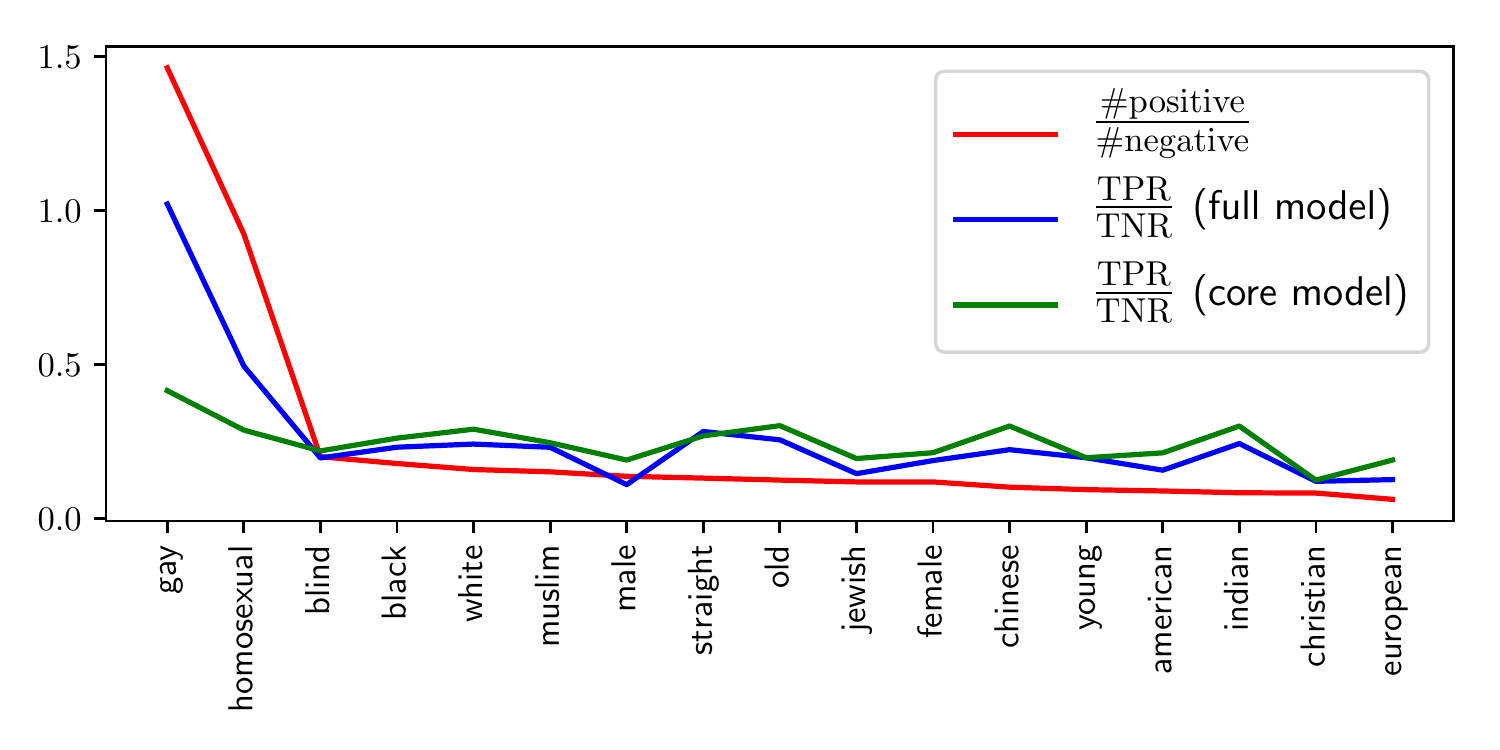}
\caption{\label{fig:17} The identity terms used as spurious features. 
The Red line indicates the ratio of positive comments containing the identity terms over the negative comments containing the identity term. 
The difference between the ratio of TPR/TNR of the core and full model can be small (straight, young) or large (male, Chinese) independent of \#positive/\#negative ratio for the identity term.}
\end{figure}
\section{Experiments}
\label{sec:experiments}
\begin{table*}[t]
\centering
\scalebox{0.77}{	\begin{tabular}{l||ccc||ccccc||c} \cmidrule[\heavyrulewidth]{2-10}
	&\multicolumn{3}{c||}{ Double-MNIST}&\multicolumn{5}{c||}{CelebA}&\multicolumn{1}{c}{Toxic-Comments}\\ \cmidrule{2-10}
	&\makecell{all}&\makecell{same\\ labels}&\makecell{different\\ labels}&\makecell {all}&\makecell{no lipstick\\or earrings}&\makecell {only\\ earrings}&\makecell {\makecell{only\\ lipstick}}&\makecell{both lipstick\\ and earrings}&\makecell{all}\\ \midrule
	group size (in percentage) &100&90&10&100&52.0&3.4&28.8&15.8 &100\\ \midrule
	full model accuracy &92.7$\pm$0.06&96.6$\pm$0.05&53.4$\pm$0.4&	83.5$\pm$0.1&85.4$\pm$0.4&33.3$\pm$0.9&79.6$\pm$0.6&95.0$\pm$0.3&	90.1 $\pm$ 0.1\\
	core model accuracy &92.1$\pm$0.06&94.8$\pm$0.06&64.1$\pm$0.3&	82.5$\pm$0.1&82.0$\pm$0.5&59.1$\pm$1.0&85.2$\pm$0.6&84.1$\pm$0.7&	89.0 $\pm$ 0.1\\
	full model robust accuracy &47.3$\pm$0.92&49.6$\pm$0.96&24.4$\pm$0.7&	66.7$\pm$0.4&58.7$\pm$0.9&33.3$\pm$0.9&79.6$\pm$0.6&76.4$\pm$0.8&	74.9 $\pm$ 0.5 \\ \midrule 
	core+RST accuracy &92.7$\pm$0.06&96.2$\pm$0.05&56.7$\pm$0.5&	83.2$\pm$0.1&83.4$\pm$0.5&58.7$\pm$1.1&84.8$\pm$0.7&85.2$\pm$0.7&	89.6 $\pm$ 0.1 \\  
	\bottomrule
\end{tabular}}
\caption{\label{tab:results_summary}
Accuracy declines as we remove spurious feature (the core model vs. the full model), however, note that this drop varies widely among different group.
The core+RST achieves similar average accuracy as the full model while having high robust accuracy and lower gap among different groups accuracy.}
\end{table*}
\begin{figure}
	\centering
	\includegraphics[width=0.3\textwidth]{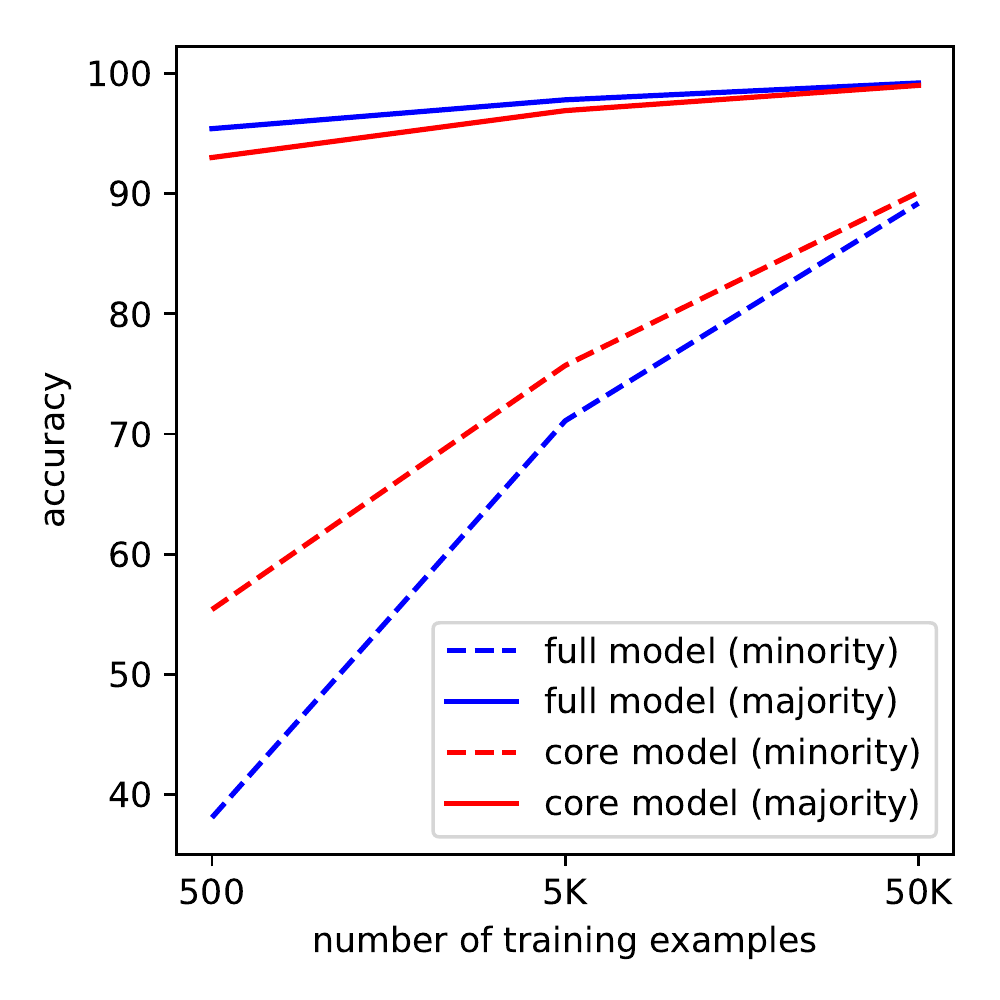}
	\caption{\label{fig:effect_of_size} 
	In Double-MNIST dataset, as we increase the number of training data the gap between the performance of core model and full model shrinks.	
	The majority groups contains data poitns where the labels of two concatenated images are the same (90\% of data), and minority group contains data points where the labels of the concatenated images are different (10\% of data).}
\end{figure}
\begin{table*}[t]
\newcommand*\rot{\rotatebox{0}}
\centering
\scalebox{0.8}{	
\begin{tabular}{ll|ccccc} \cmidrule[\heavyrulewidth]{3-7}
&& \makecell {all}&\makecell{no lipstick or earrings}&\makecell {only  earrings}&\makecell {\makecell{only  lipstick}}&\makecell{both lipstick  and earrings}\\ \midrule
\multirow{2}{120pt}{With hair color and necklace}& accuracy &83.9 $\pm$ 0.1 &85.0 $\pm$ 0.5 &37.0 $\pm$ 1.1 &81.8 $\pm$ 0.7 &94.2 $\pm$ 0.3 
\\
&robust accuracy &68.5 $\pm$ 0.4 &60.9 $\pm$ 1.0 &37.0 $\pm$ 1.1 &81.8 $\pm$ 0.7 &76.8 $\pm$ 0.8 
\\ \midrule
\multirow{2}{120pt}{With only necklace} &accuracy & 83.5$\pm$0.1&85.4$\pm$0.4&33.3$\pm$0.9&79.6$\pm$0.6&95.0$\pm$0.3
\\
&robust accuracy  &66.7$\pm$0.4&58.7$\pm$0.9&33.3$\pm$0.9&79.6$\pm$0.6&76.4$\pm$0.8
\\
\bottomrule
\end{tabular}}
\caption{\label{tab:results_multis}Two models trained on the CelebA dataset. Top model uses hair color and lipstick (a binary feature indicating if a person is wearing lipstick or not), bottom model only uses hair color. 1) the robust accuracy against hair color drop more rapidly for the bottom model. 2) The group that has lowest accuracy for the bottom model (only earrings) have a better accuracy when hair color is also used as an extra spurious feature. }
\end{table*}

We now investigate the effects of removing spurious features in non-linear models trained on real-world datasets.
Although our theory assumptions do not hold anymore, we find similar results as \refsec{ones}. 
In particular, we show that removal of a spurious feature lowers the average accuracy, has disproportionate effects on different groups, and makes the model less robust to other spurious features.  
We then show core+RST model can achieve higher accuracy than the core model while not relying on the spurious feature.
\subsection{Datasets and Setup}
\paragraph{Double-MNIST}
The MNIST dataset \citep{lecun1998gradient} consists of $60K$ images of handwritten digits between $0$ to $9$. 
We synthetically construct a new dataset from the MNIST dataset, which we call Double-MNIST.
We concatenate each image in MNIST with another random image from the same class with probability of $0.9$ and a random image from other classes with probability $0.1$.
The original image's label is the target ($y$) and the  concatenated image's label is the spurious feature ($s$). 
Note that the features that determine the target (the first image) are completely disjoint from the feature that determine the spurious feature (the second image).
We train a two-layer neural network with $128$ hidden units on this dataset.
Using $50K$ for the training data, the model achieves $98.3\%$ accuracy.
However, for our experiments (where we need unlabeled data), we used $1K$ labeled examples, $50K$ unlabeled examples, and $10K$ for test data.
\paragraph{CelebA}
The CelebA dataset \citep{liu2015deep} contains photos of celebrities along with $40$ different attributes. 
We choose wearing lipstick which indicates if a celebrity is wearing lipstick as the target and wearing earrings as the spurious feature. 
We train a two-layer neural network with $128$ hidden units on this dataset.
For our purposes in this work, we use $1K$ labeled examples, $50K$ unlabeled examples, and $10K$ for test data..
\paragraph{Toxic-Comment-Detection.}
The Toxic comment dataset is a public Kaggle dataset containing $160K$ Wikipedia comments.\footnote{\url{https://www.kaggle.com/c/jigsaw-toxic-comment-classification-challenge}}
Each comment is labeled by human raters as toxic or non-toxic, where toxicity is defined by \citet{dixon2018measuring} as ``rude, disrespectful, or unreasonable comment that is likely to make you leave a discussion.''.
We cleaned the data by replacing abbreviations (e.g., changing ``we've'' to ``we have'').
We used term frequency-inverse document frequency (tf-idf) to extract features and used a logistic regression model to train on the extracted features. 
Splitting data to the 80-20 train-test, we achieve a similar AUC ($95.8$) as reported in \citet{garg2019counterfactual}.
\citet{dixon2018measuring} provide $50$ identity terms\footnote{\url{https://github.com/conversationai/unintended-ml-bias-analysis/blob/master/unintended_ml_bias/bias_madlibs_data/adjectives_people.txt}} that a model should be robust against them. 
We choose $17$ of these identity terms that exist in positive and negative class at least $30$ times (see \reffig{17} for the list). 
We use these terms as the spurious feature and replace them with a special token for the core model. 
For the full model we did not remove these identity terms.
We choose the subset of the dataset that contains these $17$ adjectives, consisting of $11K$ examples.
We did $20-80$ train-test split, use $500$ examples as labeled data, and the rest of examples as unlabeled data.

See \reftab{datasets} for a summary of datasets.  
We run each experiment $50$ times and report the average accuracy and the standard deviation.

\subsection{Results} 
\paragraph{\textbf{Core model vs. Full model.}}  Removing the spurious feature decreases the overall accuracy in all three datasets, see the first two rows of \reftab{results_summary}.
This is in particular surprising for the Double-MNIST dataset where the target and the spurious feature get determined by disjoint set of features. 
Note that the change in accuracy varies widely among different groups (in the CelebA dataset removing the spurious wearing-earrings feature increases the accuracy for celebrities who only wear either lipstick or earrings, and in the Double-MNIST dataset removing the second image label increases the accuracy for data points where the concatenated images have the different labels).

For the Toxic-Comment-Detection dataset, for each identity term $u$, we construct two groups: one group consist of toxic comments that have $u$ in the comment text, and another group consists of non-toxic comments that have $u$ in the comment text, resulting in total of $34$ groups.
The maximum gap among accuracy of the 17 groups containing toxic comments (max-TPR-gap) drop from $70.1$ to $28.4$, and the maximum gap among accuracy of groups that are non-toxic (max-TNR-gap) drops from $20.5$ to $1.8$.
Toxic-gay and non-toxic-gay groups incurred the maximum change in accuracy ($-40$ and $+19$ respectively) due to removal of identity terms.
\reffig{17} shows the change in true positive rate, true negative rate ratio for different identity terms.

\paragraph{\textbf{Robust accuracy.}} Recall that robust accuracy is the worst-case accuracy for each data point with respect to change in the spurious feature. 
First, note that the robust accuracy of the full model (third row) is much lower than its accuracy (first row), which shows the full model's reliance on the spurious feature.
Furthermore, this reliance helps some groups (e.g., celebrities who wear earrings and lipstick) while it does not have any effect on other groups (e.g., celebrities who wear earrings but not lipstick).
Recall that the robust accuracy of the core model is exactly equal to its accuracy since it does not use the spurious feature.
Finally, in line with our theory in \refprop{rst}, the robust accuracy of the full model is always lower than the robust accuracy of the core model.

\paragraph{\textbf{Robust Self-Training.}} The forth row of the \reftab{results_summary} shows the accuracy core+RST model as explained in \refsec{rst}. 
The core+RST has a better acerage accuracy than the core model. 
Recall that Core+RST does not use the spurious feature; therefore, its robust accuracy (unlike the full model) is exactly equal to its accuracy.
In \refsec{rst}, we prove that the full model and the core+RST should have the same predictions; however, unlike our theory, we observe that the large gap among the accuracy of different groups in the full model is mitigated by the core+RST.
We have not observed a large accuracy boost using unlabeled data in the Toxic-Comment-Detection dataset.
After some error analysis we observed that as mentioned in \citet{garg2019counterfactual}, there are some examples that can be toxic with respect to some identity terms but not the others. 
Furthermore, we observed some biases in annotations in this dataset (e.g., 'username is gay' has labeled as toxic). 
As a result some accuracy drop of the core model is inherent and cannot be mitigated by unlabeled data. 

\paragraph{\textbf{Effect of training data size. }}\reffig{effect_of_size} shows that as we increase the training data size, the gap between the accuracy of the core model and full model is reduced.
In particular, it shows the error on two different groups in the Double-MNIST dataset: the majority (different labels) and minority (same labels).
When we train a model on more training examples and consequently, observing more directions in the training data, we reduce the effect of inductive bias enforced to the model by the spurious feature.
As a result, we observe a decrease in the gap between the full and the core models.

\paragraph{\textbf{Multiple spurious features.}}\reftab{multis} shows the accuracy of two models in the CelebA dataset: A model which uses hair color and wearing earrings as the spurious features and a model that only uses wearing earrings as the spurious feature.
In line with our theory in \refsec{multis}, we observe that:
1) Robust accuracy against wearing-earrings has a bigger gap from standard accuracy when we remove hair color.
2) The accuracy of people who wear lipstick but not and earrings (the group that is more vulnerable to using earrings as a spurious feature and had the lowest accuracy) drops as we remove the hair color attribute.

\section{Related work and discussion}
\label{sec:related_work}


This work is motivated by work in fairness in machine learning that aims to construct a model that is robust against changes in sensitive features. 
The techniques used in this work is similar to work in robust machine learning that tries to understand the tradeoff between the robust accuracy against perturbed inputs and the standard accuracy.
In the following, we discuss related work in these two fields. 
\paragraph{\bf Fairness in Machine Learning.} 
There are mainly two common flavors of fairness notions in machine learning concerning a sensitive feature (e.g., nationality). 
(i) The statistical notions which measure how much a model loss is different among groups according to the sensitive features 
\cite{hardt2016,khani2019mwld,agarwal2018reductions,woodworth2017}.
These notions operate at the group level, and it does not provide any guarantees at the individual level. 
(ii) Counterfactual notion of fairness which measure how much two ``similar'' individuals are incurred different losses because of their sensitive feature \cite{kusner2017,chiappa2019path,loftus2018causal,khani2020noise,kilbertus2017avoiding,garg2019counterfactual}.

This work is related to the counterfactual notion of fairness since we study the models that are robust against sensitive (spurious) features.
Note that there are many concerns and critics regarding counterfactual reasoning when sensitive features are immutable \cite{holland2003causation,freedman2004graphical}.
However, there are works that try to learn a new representation entirely uncorrelated to the sensitive feature in categorical data \cite{zemel2013,madras2018learning,louizos2015variational,mcnamara2017provably}, in vision \cite{creager2019flexibly,wang2019balanced,quadrianto2019discovering} and in natural language processing \cite{sun2019mitigating,bolukbasi2016man,zhao2018gender,zhao2018learning}.
There is a common trend in all of these works: accuracy drops when we train the same model using the new representation.

\citet{khani2020noise} show that when the core features are noisy or incomplete, the model can obtain better accuracy by using sensitive features.
As a result, removing sensitive features in these cases will lead to a drop in accuracy. 
\citet{zhao2019inherent} show that if groups according to the sensitive attribute have different base rates (probability of $y=1$ is different among different groups), it is impossible to learn the optimum classifier from a representation uncorrelated with the sensitive features.
\citet{dutta2020ista} show how biased dataset lead to trade-off in representation learning.
This work shows that under a very favorable condition, still removing spurious features changes the inductive bias of the model, which results in different performance on average and over different groups.
We believe that as overparametrized models such as deep learning models becomes more prevalent, studying inductive biases and their effect on groups gain more importance.  




\paragraph{\bf Robustness in Machine Learning.}
Since the initial demonstration of adversarial examples (examples generated by changes in images that are not perceptible by humans but change the prediction of models) \cite{szegedy2014intriguing,goodfellow2015explaining}, there have been many attempts on achieving a robust model against these perturbations \cite{madry2017towards,goodfellow2015explaining}.
However, making models robust to adversarially input perturbation comes with the cost of a drop in accuracy.
There have been many explanations regarding this drop in accuracy \cite{tsipras2019robustness,tsipras2018there,zhang2019theoretically,fawzi2018analysis,nakkiran2019adversarial}.
The recent line of work on ``double descent'' \citep{bartlett2019benign,gan2017equivalence,nakkiran2019deep,belkin2019two,mei2019generalization} has shed some lights on this trade-off and show that unlike conventional under-parameterized regime in the new over-parameterized regime, more data and fewer parameters might result in a worse error.

The closest work to us is \citet{raghunathan2020understanding}, where they show augmenting the training data with adversarially perturbed inputs changes the inductive bias and consequently can hurt the accuracy.
In contrast to their work, we focused on a specific spurious feature (with a known correlation with other features) instead of an arbitrary data perturbation set for data augmentation.  
As a result, we were enabled to analyze the drop in accuracy in a more interpretable way and show how removing the spurious feature affects different groups.
There is also another difference between our work and work on robustness to input perturbation.
In our work, by removing $s$, we have a robust model independent of the data distribution.
However, robustness to input perturbation always depends on the  distribution.

\section{Conclusion}
In this work, we first showed that overparameterized models are incentivized to use spurious features in order to fit the training data with a smaller norm. Then we demonstrated how removing these spurious features altered the model's inductive bias.
Theoretically and empirically, we showed that this change in inductive bias could hurt the overall accuracy and affect groups disproportionately.
We then proved that robustness against spurious features (or error reduction by removing the spurious features) cannot be guaranteed under any condition of the target and spurious feature.
Consequently,  balanced datasets do not guarantee a robust model and practitioners should consider other features as well. 
Studying the effect of removing noisy spurious features is an interesting future direction. 

%



\paragraph{Reproducibility.} All code, data and experiments for this paper are available on the CodaLab platform at \url{https://worksheets.codalab.org/worksheets/0x6d343ebeabd14571a9549fbf68fd28a4}.

\paragraph{Acknowledgments.} This work was supported by Open Philanthropy Project Award.
We would like to thank Michael Xie, Ananya Kumar, Rishi Bommasani, and the anonymous reviewers for useful feedback.

\bibliography{refdb/all}
\bibliographystyle{plainnat}

\appendix
\onecolumn
\linespread{2}
\renewcommand{\thesection}{\Alph{section}}
\section{Missing proofs}
\label{sec:proofs}
\biasVarianceMsMnos*

\begin{proof}
Recall, the parameters of $\mnos$ can be obtained by solving the following optimization problem.
\begin{align}
\hat \theta^{-s} = \argmin_{\theta} \| \theta\|_2^2\\
s.t. \quad Z\theta =y,
\end{align} 
and the parameters of $\ms$ is obtained by the following optimization problem.
\begin{align}
\hat\theta^{+s}, \ws = \argmin_{\theta, w} \| \theta\|_2^2 + w^2\\
s.t. \quad Z\theta + sw =y,
\end{align}
where $\ws$ denotes the weight associate with the feature $s$.

For the first scenario \refeqn{nos} we have: 
\begin{align}
\theta^{-s} = Z^\top (ZZ^\top)^{-1}y = \Pi \theta^\star 
\end{align}
In the second scenario, when we use feature $s$ we have:
\begin{align}
\label{eqn:w_sol}
\theta^{+s} = Z^\top (ZZ^\top)^{-1} (y-sw) = \Pi \theta^\star -\Pi\beta^\star \ws
\end{align}
Writing the L2-norm of the estimate and taking the derivative with respect with $\ws$, we have:
\begin{align}
\frac{\partial\|\theta^{+s}\| + \partial \ws^2}{\partial \ws} =& \frac{\partial (y-sw)^\top (ZZ^\top)^{-1}Z^\top Z^\top (ZZ^\top)^{-1} (y-sw)}{\partial \ws} + 2w\\
=&2ws^\top (ZZ^\top)^{-1}s - 2y^\top (ZZ^\top)^{-1}s+2w
\end{align} 
therefore, the optimum $\ws$ is:
\begin{align}
\label{eqn:w_comp}
\ws = \frac{y^\top (ZZ^\top)^{-1}s}{1+s^\top (ZZ^\top)^{-1}s} = \frac{{\theta^\star}^\top \Pi \beta^\star}{1+\beta^\star \Pi \beta^\star} 
\end{align}

We can now compute the error for each model:
\begin{align}
\Err (\mnos) =& \E [\tstar^\top z - \E [\theta^{-s}]^\top z ]^2 \\ =& {\theta^\star}^\top(I-\Pi)\Sigma(I-\Pi)\theta^\star \\
\Err (\ms) =& \E \pb {{\theta^\star}^\top z - {\theta^{+s}}^\top z - ws}^2 \\
=&\E \pb {{\theta^\star}^\top z - (\Pi \theta^\star)^\top z + \ws(\Pi \beta^\star)^\top z - \ws{\beta^\star}^\top z}^2\\
=&\E \pb {{\theta^\star}^\top (I-\Pi) z -\ws  \beta^\star(I-\Pi) z}^2\\
 =& {\theta^\star}^\top(I-\Pi)\Sigma(I-\Pi)\theta^\star \\
 &+ \ws^2 {\beta^\star}^\top(I-\Pi)\Sigma(I-\Pi)\beta^\star  -2\ws {\theta^\star}^\top(I-\Pi)\Sigma(I-\Pi)\beta^\star
\end{align}
\begin{align}
	\Err (\ms ) - \Err (\mnos) = \ws^2 {\beta^\star}^\top(I-\Pi)\Sigma(I-\Pi)\beta^\star  -2\ws {\theta^\star}^\top(I-\Pi)\Sigma(I-\Pi)\beta^\star
	\label{eqn:difference}
\end{align}
In order for \refeqn{difference} to be negative $\ws$ should be non zero and $\ws$ and $\tstar^\top (I-\Pi)\Sigma(I-\Pi)\bstar$ should have the same sign. 
Therefore,
$\Err (\ms) < \Err (\mnos)$, iff:
\begin{align}
	\sign (\bstar^\top \Pi \tstar) &= \sign (\tstar^\top (I-\Pi)\Sigma(I-\Pi)\bstar)\\
 \pab {\frac{{\beta^\star}^\top \Pi \theta^\star}{{1+\beta^\star}^\top \Pi \beta^\star}} &< \pab {\frac{2{\beta^\star}^\top (I-\Pi)\Sigma (I-\Pi) \theta^\star}{{\beta^\star}^\top (I-\Pi) \Sigma (I-\Pi) \beta^\star}}
 \end{align}
 
\end{proof}

\needZone*

\begin{proof}
We find three directions ($a_1, a_2, a_3$) in the feature space where the dot product between $\bstar$ and $\tstar$ is non-zero in the first direction, positive in the second direction, and negative in the third direction.
We then construct a training dataset in the first direction, and construct two test datasets in the other two directions. We then use \refprop{whenRemoveSHurts} to prove the corollary.

Let $b \in \R^d$ be a vector that is orthogonal to both $\bstar$ and $\tstar$.
We define $a_2 = (\frac{\tstar}{\|\tstar\|} + \frac{\bstar}{\|\bstar\|} +2b)$ and $a_3 = (\frac{\tstar}{\|\tstar\|} - \frac{\bstar}{\|\bstar\|})$. 
We now find $a_1$ such that it is orthogonal to $a_2$ and $a_3$ (i.e., $a_1 ^\top a_2=0$ and $a_1 ^\top a_3 =0$) but not orthogonal to $\tstar$ and $\bstar$ (i.e., $a_1^\top \bstar \neq 0$ and $a_1^\top \tstar \neq 0$).

\begin{align}
	&a_1^\top (\frac{\tstar}{\|\tstar\|} + \frac{\bstar}{\|\bstar\|} +2b) = 0\\
	&a_1^\top (\frac{\tstar}{\|\tstar\|} - \frac{\bstar}{\|\bstar\|}) =0\\
	&a_1^\top \tstar \neq 0\\
	&a_1^\top \bstar \neq 0
\end{align}
Fixing $b$ and for a $x > 0$ we can
simplifying the above equations and finding a solution by solving:
\begin{align}
&a_1^\top b = -x\\
&a_1^\top \frac{\tstar}{\|\tstar\|} = x\\
&a_1^\top \frac{\bstar}{\|\bstar\|} = x
\end{align}

Since $\tstar,\bstar,b$ are not parallel the above problem has  a solution for any $x$.

We now construct $Z$ to consist of $n$ examples of $a_1$, and construct $Z'$ and $Z''$ to have $n$ examples of $\frac{1}{n}a_2$ and $\frac{1}{n}a_3$ respectively.
Note that instead of copying the same examples to both train and test from we can add examples from directions that are orthogonal to $a_1,a_2,a_3$, and $\tstar, \bstar$.

Note that in this case $\Pi = \frac{a_1a_1^\top}{a_1^\top a_1}$.
As a result $\sign (\tstar \Pi \bstar) = \sign (x^2)$ is positive. 
However, the dot product of \tstar{} and \bstar{} is negative when projected on $a_3$.

\begin{align}
\label{eqn:diff_sign}	
\tstar^\top (I - a_1a_1^\top) (a_3 a_3^\top) (I- a_1a_1^\top) \bstar =&  \tstar^\top (a_3^\top a_3) \bstar\\
=& \tstar^\top (\frac{\tstar}{\|\tstar\|} - \frac{\bstar}{\|\bstar\|}) (\frac{\tstar}{\|\tstar\|} - \frac{\bstar}{\|\bstar\|})^\top \bstar\\
=&(\|\tstar\| - \frac{\tstar^\top \bstar}{\|\bstar\|})(\frac{\tstar^\top \bstar}{\|\bstar\|} - \|\bstar\|)\\
=&-\|\tstar\|\|\bstar\| \left(1 - \frac{\tstar^\top \bstar}{\|\tstar\|\|\bstar\|}\right)^2\\<&0
\end{align}
According to \refeqn{diff_sign} and \refeqn{sign}, $\Err (\ms ) > \Err (\mnos)$ on $Z''$.

We now show that the dot product of $\tstar,\bstar$ projected on $a_2$ is positive.
\begin{align} 
	\label{eqn:same_sign}	
	\tstar^\top (I - a_1a_1^\top) (a_2 a_2^\top) (I- a_1a_1^\top) \bstar =&  \tstar^\top (a_2^\top a_2) \bstar\\
	=& \tstar^\top (\frac{\tstar}{\|\tstar\|} + \frac{\bstar}{\|\bstar\|} +2b) (\frac{\tstar}{\|\tstar\|} + \frac{\bstar}{\|\bstar\|} +2b)^\top \bstar\\
	=&(\|\tstar\| + \frac{\tstar^\top \bstar}{\|\bstar\|})(\|\bstar\| + \frac{\tstar^\top \bstar}{\|\bstar\|})\\
	=&\|\tstar\|\|\bstar\| \left(1 + \frac{\tstar^\top \bstar}{\|\tstar\|\|\bstar\|}\right)^2\\>&0.
\end{align}

In \refeqn{same_sign} we showed that \refeqn{sign} holds for $Z'$, in addition we should also show that \refeqn{magn} also holds.
\begin{align}
	\pab {\frac{{\beta^\star}^\top \Pi \theta^\star}{{1+\beta^\star}^\top \Pi \beta^\star}} &\le  \pab {\frac{2{\beta^\star}^\top (I-\Pi)\Sigma (I-\Pi) \theta^\star}{{\beta^\star}^\top (I-\Pi) \Sigma (I-\Pi) \beta^\star}}\\
		\pab {\frac{{\beta^\star}^\top \frac{a_1a_1^\top}{a_1^\top a_1} \theta^\star}{{1+\beta^\star}^\top \frac{a_1a_1^\top}{a_1^\top a_1} \beta^\star}} &\le  \pab {\frac{2{\beta^\star}^\top a_2a_2^\top \theta^\star}{{\beta^\star}^\top a_2a_2^\top \beta^\star}}\\
	\frac{x^2}{a^\top a+x^2} &\le \frac{2(\|\tstar\| + \frac{\bstar^\top \tstar}{\|\bstar\|})}{\frac{\tstar^\top \bstar}{\|\tstar\|} + \|\bstar\|}\\
		\frac{x^2}{a^\top a+x^2} &\le \frac{2\|\tstar\|}{\|\bstar\|},
\end{align}
since we assumed $d \ge 4$ we can choose a small $x$ and increase the norm of $a$ on other directions to satisfy this inequality. 

Now we show that if $\tstar = c \bstar$ then $\Err (\ms ) \le \Err (\mnos)$. 
First note that if $\ws = 0$ then $\Err (\ms ) = \Err (\mnos)$. 
Furthermore, if $\tstar^\top (I-\Pi) \Sigma (I-\Pi) \bstar =0$ then $\bstar^\top (I-\Pi) \Sigma (I-\Pi) \bstar = 0$ which implies $\Err (\ms ) = \Err (\mnos)$. 
We now assume $\ws \neq 0$, and $\bstar^\top (I-\Pi) \Sigma (I-\Pi) \bstar \neq 0$, and we show that \refeqn{sign} and \refeqn{magn} always hold.

\begin{align}
&\sign (\tstar \Pi \bstar ) = \sign (c \bstar \Pi \bstar) = \sign (c)\\ 	&\sign (\tstar (I-\Pi) \Sigma (I-\Pi) \bstar ) = \sign (c\bstar (I-\Pi) \Sigma (I-\Pi) \bstar ) = \sign (c)
\end{align}
\begin{align}
	 \pab {\frac{{\beta^\star}^\top \Pi \theta^\star}{{1+\beta^\star}^\top \Pi \beta^\star}} &< \pab {\frac{2{\beta^\star}^\top (I-\Pi)\Sigma (I-\Pi) \theta^\star}{{\beta^\star}^\top (I-\Pi) \Sigma (I-\Pi) \beta^\star}}\\
	 \frac{\bstar^\top \Pi \bstar}{1 + \bstar^\top \Pi \bstar}\pab{c} &< 2\pab {c}
\end{align}

\end{proof}

%

\needZtwo*

\begin{proof}
	
	For any vector $u \in \BR^{k\times d}$ let $\bar u \in \BR^{k\times 2}$ denote its first $2$ columns.
	Similarly for $v \in \BR^d$, let $\bar v \in \BR^2$ denote its first $2$ elements.
	
Define $\bar \tstar = [1,1]$ and $\bar \bstar = [1,0]$. 
	We construct $\bar Z \in \BR^{n\times 2}$ such that $Z\tstar= Y$ and $Z\bstar = S$.
	In particular, $\bar Z = [S, Y-S]$. 
	Assign the rest of elements in $Z$ to be zero, $Z=[\bar Z, 0]$. 
	As before, let $\Pi$ be the column space of training data, $\Pi = Z^\dagger Z$.
	Since we assumed $S$ and $Y$ are not parallel, $\Pi$ is a zero matrix with a $2\times 2$ identity matrix and the top left (If all the rows of $Z$ are parallel then it implies that $S$ and $Y-S$ are parallel which means $S$ and $Y$ are parallel).

	Now we find $a,a' \in \BR^{d}$ such that 
(i) $a\tstar =0$ and $a'\tstar =0$, (ii)
$a\bstar =0$ and $a'\bstar =0$, and (iii) conditions on  \refprop{whenRemoveSHurts} holds for $Z+A$ but not for $Z+A'$, where $A, A' \in \R^{n \times d}$ are $n$ copy of $\frac{1}{n}a,\frac{1}{n}a'$ respectively.

The first two conditions ensure that adding $a,a'$ to the rows of $Z$ does not change the targets or spurious features.
Define $Z' = Z + A$ and $Z'' = Z + A'$.

Since we assumed $a\tstar=0$ we can write: 
\begin{align}
		0=a\tstar = a (I-\Pi) \tstar + a \Pi \tstar\implies
		a(I-\Pi) \tstar = - a \Pi \tstar
	\end{align}
	Also note that $(I-\Pi)Z=0$. Now we rewrite the RHS of \refeqn{sign}.
	\begin{align}
		\tstar^\top (I-\Pi) (Z+A)^\top (Z+A) (I-\Pi) \bstar =& 	\tstar^\top (I-\Pi) A^\top A (I-\Pi) \bstar \\ 
		=& 	\tstar^\top (I-\Pi) a^\top a (I-\Pi) \bstar \\ 
		=& \tstar^\top \Pi a^\top a \Pi \bstar
		\\
		=& \bar \tstar \bar a^\top \bar a \bar \bstar,
	\end{align}
where the last inequality holds since $\Pi$ is a zero matrix with a $2\times 2$ identity matrix at the top left corner.

Now define $\bar a = (\frac{\bar \tstar}{\|\bar \tstar\|} + \frac{\bar \bstar}{\|\bstar\|})^\top$ and
	$\bar a' = (\frac{\bar \tstar}{\|\bar \tstar\|} - \frac{\bar \bstar}{\|\bstar\|})^\top$.
Let $\vec{0} \in \R^{d-4}$ denote an all zero vector (when $d=4$ this is empty). 
Define
	$\tstar=[\bar \tstar,\vec{0}, 1, 0]$ and $\bstar =[\bar \bstar,\vec{0}, 0, 1]$.
	We now show that  $a = [\bar a, \vec{0},-\bar a\bar\tstar, -\bar a\bar\bstar]
	$	and $a' = [\bar a', \vec{0}, -\bar a'\bar \tstar, -\bar a'\bar \bstar]
	$
	satisfy all three conditions.
	For the first two conditions we have:
 $a\tstar = [\bar a, \vec{0}, -\bar a \bar \tstar, -\bar a \bar \bstar][\bar \tstar,\vec{0},1,0]^\top = \bar a \tstar - \bar a \bar \tstar = 0$, similarly, $a\bstar = a'\tstar = a'\bstar =0$.
Similar to \refeqn{same_sign} and \refeqn{diff_sign} we can show the \refeqn{sign} holds for $Z'$ but not $Z''$ 
Now we only need to show:
	
	\begin{align}
			\pab {\frac{{\bstar}^\top \Pi \tstar}{{1+\bstar}^\top \Pi \bstar}} &\le  \pab {\frac{2{\bstar}^\top (I-\Pi)\Sigma (I-\Pi) \tstar}{{\bstar}^\top (I-\Pi) \Sigma (I-\Pi) \bstar}}\\
			\pab {\frac{{\bar \bstar}^\top  \bar \tstar}{{1+\bar \bstar}^\top \bar \bstar}} &\le  \pab {\frac{2{\bar \bstar}^\top \bar a \bar a^\top \bar \tstar}{{\bar \bstar}^\top \bar a \bar a^\top \bar \bstar}}\\
		\pab {\frac{{\bar \bstar}^\top  \bar \tstar}{{1+\bar \bstar}^\top \bar \bstar}} &\le  \pab {\frac{2{\bar \bstar}^\top (\frac{\bar \tstar}{\|\bar \tstar\|} + \frac{\bar \bstar}{\|\bstar\|})(\frac{\bar \tstar}{\|\bar \tstar\|} + \frac{\bar \bstar}{\|\bstar\|})^\top \bar \tstar}{{\bar \bstar}^\top (\frac{\bar \tstar}{\|\bar \tstar\|} + \frac{\bar \bstar}{\|\bstar\|})(\frac{\bar \tstar}{\|\bar \tstar\|} + \frac{\bar \bstar}{\|\bstar\|})^\top \bar \bstar}}\\
		\frac{1}{2} &\le \frac{2\|\tstar\|}{\|\bstar\|}=2\sqrt{2}
	\end{align}

We now show if $Y=cS$ then full model always outperforms the core model.
Let $Z\in \R^{n\times d}$ be the training data and $Z' \in \R^{n \times d}$ be the test data.
\begin{align}
	Z\tstar - cZ\bstar = Y - cS = 0 \implies Z(\tstar - c\bstar)= 0 \implies Z^\dagger Z (\tstar - c\bstar) =0 \implies \Pi \tstar = c\Pi \bstar
\end{align}

Let $A=Z' - Z$. Due to our assumption we have $A\tstar = A\bstar = 0$, therefore, 
\begin{align}
	0=A\tstar = A (I-\Pi) \tstar + A \Pi \tstar\implies
	A(I-\Pi) \tstar = - A \Pi \tstar
\end{align} 

We can now show:
	\begin{align}
	\tstar^\top (I-\Pi) (Z+A)^\top (Z+A) (I-\Pi) \bstar =& 	\tstar^\top (I-\Pi) A^\top A (I-\Pi) \bstar \\ 
	=& \tstar^\top \Pi A^\top A \Pi \bstar
\end{align}

We can write conditions on \refprop{whenRemoveSHurts} as follows:
\begin{align}
	\sign (c\bstar^\top \Pi \bstar) = \sign (c \bstar^\top \Pi AA^\top \bstar)
\end{align}
which is always true, we can write the second condition:
\begin{align}
	\pab {\frac{{\bstar}^\top \Pi \tstar}{{1+\bstar}^\top \Pi \bstar}} &\le  \pab {\frac{2{\bstar}^\top (I-\Pi)\Sigma (I-\Pi) \tstar}{{\bstar}^\top (I-\Pi) \Sigma (I-\Pi) \bstar}}\\
	\pab {\frac{c{\bstar}^\top \Pi \bstar}{{1+\bstar}^\top \Pi \bstar}} &\le  \pab {\frac{2c{\bstar}^\top \Pi AA^\top \Pi \bstar}{{\bstar}^\top \Pi AA^\top \Pi \bstar}}\\
\frac{{\bstar}^\top \Pi \bstar}{{1+\bstar}^\top \Pi \bstar} &\le  2\\
\end{align}
Which is again always true.
Therefore, for any choice of $Z$ and $Z'$ full model always have lower or equal error to the core model on $(Z,S,Y)$.

\end{proof}

As a special case, consider the setup where features that determine $s$ and $y$ are disjoint, (i.e., $\bstari\neq 0 \implies \theta_i^\star=0$ and $\theta_i^\star \neq 0 \implies \bstari=0$).
First of all, note that it is still possible that $\tstar$ and $\bstar$ have a large correlation on the observed direction (i.e., projected on $\Pi$), which results in  $\ws \neq 0$.
In this case, if we assume the covariance matrix at the test time is identity, then removing $s$ always helps. 
\begin{restatable}{corollary}{specialCase}
	If true features that determine $s$ and $y$ are disjoint, $\beta_i^\star \neq 0 \implies \theta_i^\star=0$ and $\theta_i^\star \neq 0 \implies \beta_i^\star =0$, and the test time covariance matrix $\Sigma = I$ then removing $s$ always reduce the error.
\end{restatable}

\begin{proof}
	\begin{align}
		\bstar ^\top \tstar = \bstar^\top \Pi \tstar + \bstar^\top (I-\Pi)\tstar=0
	\end{align}
	WLOG, we can assume $\bstar^\top \Pi \tstar > 0$.
	Using \refeqn{when_nos_helps} and the assumption that $\Sigma = I$, we know that removing $s$ increases the error if $\bstar^\top (I-\Pi)\tstar > 0$ which is impossible since we assume $\bstar^\top\tstar =0$. 
\end{proof}

%
%

\robustErrorCompared*
\begin{proof}
 We will show that for each data point $(z,s)$ we can shift $s$ in its perturbation set to $s'$ such that $\ms (z,s') = \mnos (z,s)$, which implies $\robErr (\ms) \ge \Err (\mnos) = \robErr (\mnos)$.
Define $s' = ((I - \Pi) \bstar)^\top z$
 \begin{align}
 	\ms (z,s') =& (\Pi \tstar)^\top z - \ws ((I - \Pi) \bstar)^\top z + ws'\\
 	=& (\Pi \tstar)^\top z - \ws ((I - \Pi) \bstar)^\top z + \ws ((I - \Pi) \bstar)^\top z\\
 	=&(\Pi \tstar)^\top z\\
 	=&\mnos(z,s)
 \end{align}
 Now it suffices to prove $s' \in \sVal$.
 \begin{align}
 |s'| =& |((I - \Pi) \bstar)^\top z|\\
  =& |\bstar^\top (I-\Pi) z|\\
  \le& \gamma \|\bstar\|_*
 \end{align}
 where in the last inequality we used the fact that $(I-\Pi)$ is a projection matrix, and multiplying a vector with a projection matrix results in a vector with a smaller norm. 
 \end{proof}

\multiSFeatures*
\begin{proof}
Similar to \refeqn{w_sol} we have:
\begin{align}
	\ts = \Pi \tstar - \Pi \bstar \ws
\end{align}
Similar to \refeqn{w_comp} we can compute the optimum value for $\ws$ as follows:
\begin{align}
\frac{\partial\|\ts\| + \partial \sum_j w_j^2} {\partial w_i} =& \frac{\partial\left(\sum_j{w_iw_j {\bstari}^\top \Pi \bstarj} - 2w_i{\tstar}^\top \Pi \bstari\right)}{\partial w_i} + 2w_i\\
=&\sum_{j\neq i} 2w_j {\bstari}^\top \Pi \bstarj +2w_i {\bstari}^\top \Pi \bstari   - 2{\tstar}^\top \Pi \bstari+2w_i
\end{align} 
Therefore, the optimum value for $w_i$ is:
\begin{align}
w_i = \frac{\tstar^\top \bstari - \sum\limits_{j\neq i} w_j \bstari^\top \Pi \bstarj}{1+ \bstari^\top \bstari}
\end{align}
\end{proof}

\rst*
\begin{proof}
We show how to derive the esimtated parameters for the core model + RST as introduced in \refsec{rst}.
Recall that we are interested in the following optimization problem 
\begin{align}
	\label{eqn:self_training}
	\tnosST = &\argmin_\theta \|\theta\|_2^2\\
	\text{s.t}\quad &Z\theta = Y\\
	&Z_u \theta = Z_u \ts + S_uw 
\end{align}

Substituting $\ts$, $Y$, and $S_u$ in terms of $\bstar$ and $\tstar$ we have:
\begin{align}
	\tnosST = &\argmin_\theta \|\theta\|_2^2\\
\text{s.t}\quad &Z\theta = Z\tstar \\
&Z_u \theta = Z_u (\Pi \tstar - \ws\Pi \bstar) + Z_u\bstar \ws = Z_u (\Pi \tstar - \ws (I-\Pi)\bstar) \label{eqn:zu} 
\end{align}

as explained in \refeqn{w_comp}, we have: $\ws=\frac{{\bstar}^\top \Pi \tstar}{{1+\bstar}^\top \Pi \bstar}$ .
Since we assumed we have $m > d$ unlabeled examples then solving \refeqn{zu} results in 
\begin{align}
	\tnosST &= \Pi \tstar + \ws(I-\Pi) \bstar
\end{align}
We now prove $\mnosST$ makes the same predictions as $\ms$.
\begin{align}
\mnosST(z,0) = (\Pi \tstar + \ws(I-\Pi) \bstar)^\top z = (\Pi \tstar - \ws\Pi \bstar)^\top z + \ws\bstar^\top z =  (\Pi \tstar - \ws\Pi \bstar)^\top z + ws = \ms (z,s)
\end{align}
\end{proof}

\end{document}